\documentclass[journal]{IEEEtran}
\usepackage{ifpdf}
\usepackage[pdftex]{graphicx}
\usepackage[cmex10]{amsmath}
\usepackage{graphicx}
\usepackage{amsmath}
\usepackage{amssymb}
\usepackage{subfigure}
\usepackage{epstopdf}
\usepackage{multirow}
\usepackage{array}
\usepackage{booktabs}
\usepackage{stfloats}
\usepackage{amsmath,amsthm}
\ifCLASSINFOpdf
\else
\fi

\begin{document}
%
\title{Local Color Contrastive Descriptor for Image Classification}
%
%
%

\author{Sheng~Guo,~\IEEEmembership{ Student~Member,~IEEE,}
        Weilin~Huang,~\IEEEmembership{Member,~IEEE,}
        and~Yu~Qiao,~\IEEEmembership{Senior~Member,~IEEE}
\thanks{Sheng Guo, Weilin Huang and Yu Qiao are with Shenzhen Institute of Advanced Technology, Chinese Academy of Sciences, Shenzhen, China,
and Multimedia Laboratory, Chinese University of Hong Kong, Hong Kong.

 E-mail: \{sheng.guo,wl.huang,yu.qiao\}@siat.ac.cn.}}

\maketitle

\begin{abstract}
Image representation and classification are two fundamental tasks towards multimedia content retrieval and understanding. The idea that shape and texture information (e.g. edge or orientation) are the key features for visual representation is ingrained and dominated in current multimedia and computer vision communities. A number of low-level features have been proposed by computing local gradients (e.g. SIFT, LBP and HOG), and have achieved great successes on numerous multimedia applications. In this paper, we present a simple yet efficient local descriptor for image classification, referred as Local Color Contrastive Descriptor (LCCD), by leveraging the neural mechanisms of color contrast. The idea originates from the observation in neural science that color and shape information are linked inextricably in visual cortical processing. The color contrast yields key information for visual color perception and provides strong linkage between color and shape. We propose a novel contrastive mechanism to compute the color contrast in both spatial location and multiple channels.   The color contrast is computed by measuring  \emph{f}-divergence between the color distributions of two regions. Our descriptor enriches local image representation with both color and contrast information. We verified experimentally that it can compensate strongly for the shape based descriptor (e.g. SIFT), while keeping computationally simple. Extensive experimental results on image classification show that our descriptor improves the performance of SIFT substantially by combinations, and achieves the state-of-the-art performance on three challenging benchmark datasets. It improves recent Deep Learning model (DeCAF) \cite{decaf2013} largely from the accuracy of $40.94\%$ to  $49.68\%$ in the large scale SUN397 database. Codes for the LCCD will be available.
\end{abstract}


%
\IEEEpeerreviewmaketitle

\section{Introduction}
Image representation has long been an active yet challenging topic in multimedia community. It is a fundamental task for image content understanding, and plays a crucial role on the success of numerous image/video related applications, such as image categorization \cite{sadeghi2012latent,parizi2012reconfigurable,wu2011centrist,gao2014concurrent}, object detection/recogntion \cite{nister2006scalable,van2011segmentation,lei2011face}, action recognition \cite{tahir2013robust,wang2014semi}, and image segmentation \cite{arbelaez2011contour,ren2013image,taylor2013towards}.  For the last two decades, a large amount of research efforts have been devoted to
designing an efficient local descriptor for image/video representation. The state-of-the-art descriptors are mostly based on shape descriptions, such as edges, corners or gradient orientations, while discarding color information. Typical examples along this line include  Scale Invariant Feature Transform (SIFT) \cite{1}, Local Binary Pattern Descriptor (LBP) \cite{2, Huang2015}, Histogram of Orientated Gradient (HOG) \cite{4}, Region Covariance Descriptor (RCD) \cite{Tuzel2006, Huang2013}, and PixNet visual features \cite{pourian2015pixnet}. They have been widely applied for numerous multimedia and image/vision applications with great success achieved. Most of these descriptors are designed for gray images. Obviously, the concept that local shape or gradient information are the key features for image representation is ingrained and dominated. It makes sense in the way that people can easily understand the contents or actions in a black and white movie, like a Charlie Chaplin film, without knowing its true color.

\begin{figure*}
\centering
\includegraphics[height=2in,width=4.8in,angle=0]{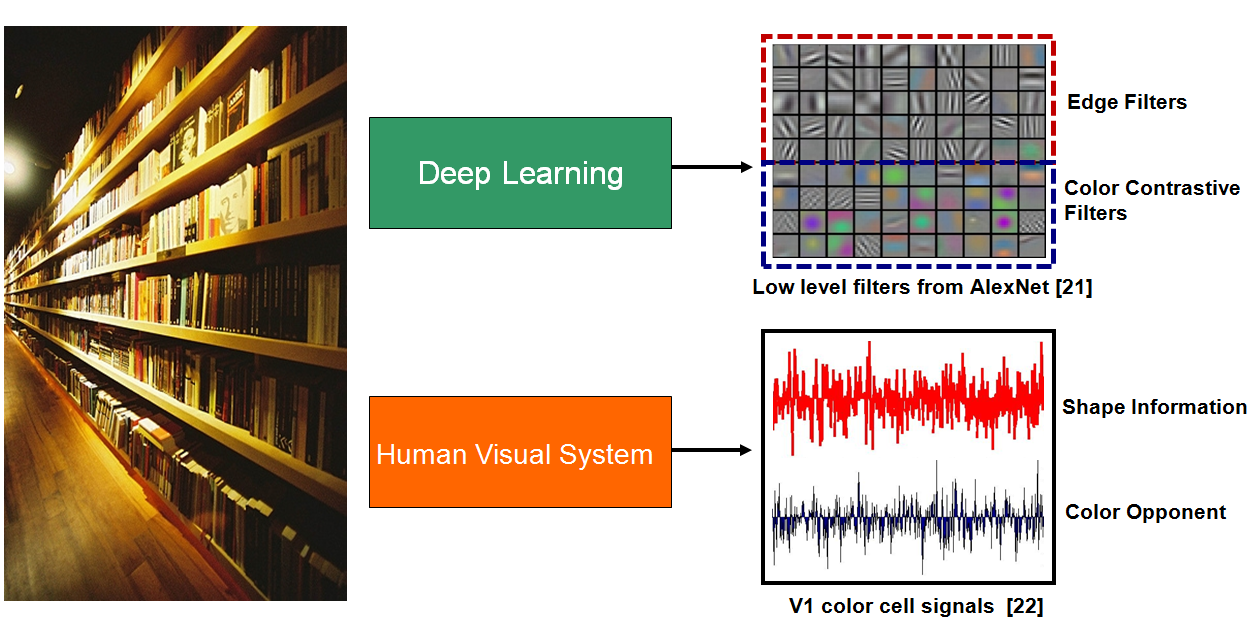}
\caption{Top: the low-level filters (from the first convolutional layer) learned by Deep Convolutional Neural Network \cite{Alex2012imagenet}. Bottom: shape and opponent information in human visual system \cite{yamins2014performance}. The low-level filters of the deep CNN contain both edge/shape filters and color constrictive information, which are highly consistent with the shape and opponent information in human visual system. Both of them validate our clam theocratically that color contrast plays a crucial role in image description.}
\label{fig:12}
\end{figure*}
 However, recent observations from visual neuroscience indicate that shape/form information is not the only visual property of objects and surfaces, but rather color and shape are inextricably linked as properties of scene in visual perception and visual cortical processing \cite{31}\cite{32}\cite{33}. In order to account for color information for image representation, Mindru \textit{et.al.} \cite{mindru2004moment} proposed a combined color and shape feature of local patch based on color moments, which is shown to be invariant to illuminance changes. In \cite{34}, the description of local feature is extended with color information in an effort to increase its robustness against photometric changes and varying image quality. To increase the photometric invariance and discriminative power, a number of color descriptors based on both color histogram and SIFT are systemically reviewed and evaluated in \cite{35}, including the color SIFT. However, the application of color information  for image description has received relatively much less attentions in multimedia research \cite{34}\cite{35}, mainly due to the large amount of variations in real-world scene which may significantly increase the difficulty of robustly measuring color information.

 The goal of this paper is to enrich local image representation with color and contrast information by presenting an efficient local descriptor. The key issue lies in how to robustly extract efficient color contrast information which could effectively interact with and strongly compensate for the shape information. It has been observed in visual neuroscience that color perception of a region sometime is more dependent on color contrast at the boundary of the region than  on the spectral reflectance of the region \cite{31}.  Color contrast can have a major effect on color perception, and color and shape interact through the spatial layout of surrounding them relative to a target region \cite{31}.

 Motivated by these biological findings, we develop a novel local descriptor, named as Local Color Contrastive Descriptor (LCCD). Our LCCD descriptor utilizes \emph{f}-divergence to effectively measure the color contrast, which enables it with strong ability to robustly describe local contents of the image. The \emph{f}-divergence yields a class of measures between local distributions, and has been successfully used as features for robust speech recognition \cite{qiao2009affine}. The LCCD differs distinctly from most current local color descriptors, which often extract color features from each region independently, while discarding important correlation information between neighboring regions. The region-based contrastive information enhances spatial locality of the LCCD, and increases robustness against noises which are easily caused by single-pixel operations. Finally, we demonstrate efficiency and effectiveness of the LCCD descriptor for image classification. The main contribution of the paper is summarized as bellow.

 1) We  propose a novel local descriptor, the LCCD, by leveraging the color contractive feature. We develop a novel mechanism to measure color contrast of the image region that detects local contrastive information in both spatial location and multiple color channels. We find  that the proposed color contrast is highly consistent with the structure of low-level filters learned by deep convolutional neural network, as shown in Figure 1.

 2) We introduce \emph{f}-divergence to robustly measure the difference of color distributions between local regions. It has been shown that the \emph{f}-divergence measure is invariant to invertible transformations \cite{qiao2008f}. We leverage this appealing property to enable the LCCD with strong robustness against multiple local image distortions.

3) We propose subspace extension to compute the $f$-divergence measure between different feature histograms. This improvement enables our descriptor with stronger capability for capturing more detailed information from the images, which increases its discriminative power considerably.


4) We show experimentally that the color contrastive information can strongly compensate for gradient-based SIFT by substantially improving its performance through combination. The LCCD descriptor with SIFT achieves remarkable results on three benchmarks: the MIT Indoor-67 database \cite{16}, SUN397 \cite{xiao2010sun} and PASCAL VOC 2007 standards\cite{29}, advancing the state-of-the-art results considerably.

The rest of paper is organized as follows.  Section 2 briefly reviews related studies. Details of the proposed LCCD are presented in Section 3, including descriptions of the $f$-divergence, spatially and channelly contractive features, and the subspace extension. Experimental results are presented in Section 4, followed by the conclusions in Section 5.

\section{Related Works}

    The local image descriptor has been an active research topic in image and multimedia communities in the last years.  Lowe \cite{1} proposed powerful Scale Invariant Feature Transform (SIFT) descriptor by computing a 3D histogram of gradient location and orientation. The spatial location is divided into a $4\times4$ grid and the gradient angle is quantized into eight orientations. The SIFT descriptor is computed based on appearance of the object at particular interest points, and is invariant to transform of scale and rotation. It is  robust against changes in illumination, noise, and viewpoint. Although the SIFT is powerful for image description by computing gradient features, it does not exploit the color information, leading to a less informative representation.
 %
    %
\par
 Local Binary Pattern (LBP) \cite{2} and its extensions have achieved great successes on texture description \cite{Ojala2002} and face recognition \cite{Huang2015,Ahonen2006,3}. It labels image pixels  by thresholding neighborhood of a central pixel to generate a binary string for feature representation. It has been widely applied for face and  texture recognition due to its high performance and computational simplicity. However, the LBP involves pixel-level operation to compute the binary features, which largely limits it robustness against noise and multiple local image distortions.
\par
Recent effort focuses on developing mid/hight-level models by encoding multiple low-level features for image description, such as Bag-of-Features (BoF) \cite{6},
 Object Bank (OB) \cite{7}, and Bag-of-Parts (BoP) models \cite{8}. The OB is a high-level image representation where an image is represented as a scale-invariant response map of pre-trained generic object detectors  \cite{7}. The BoP, building on the HOG feature, automatically detects distinctive parts from scene image for recognition \cite{8}.
%
 %
 Furthermore, a number of recent descriptors achieved impressive results by combining with the SIFT  and SPM model. For example, Orientational Pyramid Matching (OPM) \cite{10} utilizes 3D orientations of objects to form the pyramid and produce the pooling regions. It shows strong complementary abilities to the SIFT and SPM so that the combination of them achieved excellent performance in scene recognition. However, similar to the SIFT, these descriptors do not include either color  or local contrastive information. We will show that these features work as important complementary information for image representation.
\par

There are few works to apply color information for designing local image descriptors. Local Color Statistic (LCS) descriptor \cite{Clinchant2007} was proposed by computing the means and standard deviations of the 3 RGB channels from 16 sub-regions, which results in a 96-dim color feature. It was combined with SIFT descriptor for Fisher Vector encoding, and achieved remarkable performance for scene image classification \cite{28}. In \cite{35}, color descriptors based on color histograms and
moments, and color SIFT were proposed and discussed. It has been demonstrated experimentally that combination of multiple color descriptors and SIFT leads to significant performance improvements on image/video classification \cite{35}. Our work is related to these descriptors by leveraging the color information for image description. By contrast, we compute both color and local contrastive information of the image so that encode richer local features. Therefore, our LCCD descriptor provides a more principled approach for measuring the color information, which sets it apart from all color descriptors of the past.

\section{Local Color Contrastive Descriptor}
This section presents details of the proposed Local Color Contrastive Descriptor (LCCD). We first introduce the \emph{f}-divergence measurement which is applied for robustly computing the contrastive characteristic between two local  regions.  Then the LCCD descriptor is constructed by two types of contrastive features: local spatially-contrastive and channelly-contrastive features. The two features extract contrastive information from spatial locality and multiple color channels, respectively. Finally, a subspace extension is developed to further enhance its discriminative ability.

\subsection{\emph{F}-divergence for Contrastive Measurement}
Computing contractive information between two feature distributions of local regions servers as the basic component of the LCCD descriptor.  Traditionally, classical $L_p$ (e.g. $L_1$ or $L_2$) distance is used to compute the difference (dissimilarity) between two feature vectors. However, the family of \emph{f}-divergence has been shown to be more suitable to measure the contrastive information, due to its robustness to transformations \cite{qiao2008f}. It has been widely applied in statistical learning and information theory, and also has achieved excellent results on  robust speech recognition recently \cite{44}.

 In statistics and information theory, Csisz$\acute{a}$r \emph{f}-divergence \cite{41} (also known as Ali-Silvey distance \cite{42}) measures the difference (dissimilarity) between two distributions. Formally, $f:(0,\infty)\rightarrow R$ is a real convex function and $f(1)=0$, $p_{i}(x)$ and $p_{j}(x)$ are density functions of two distributions defined on measurable $\Re$. Then

\begin{small}
\begin{eqnarray}
D_{f}(p_{i},p_{j})=\int_{\Re}p_{j}(x)f(\frac{p_{i}}{p_{j}})dx
\end{eqnarray}
\end{small}
To insure the \emph{f}-divergence between two identical distribution is zero, $D_{f}(p,p)=0$, we set constraint $f(1)=0$ \cite{41}.  It has been proved that many well-known distances or divergences in statistics and information theory, such as KL divergence, Bhattacharyya distance, Hellinger distance, etc., can be regarded as one of special cases of the \emph{f}-divergence measure\cite{44,werner2013mixed}, depending on the choice of function \emph{f}. The detailed functions of them are presented in Table 1. The choice of the \emph{f} will be discussed as bellow.
\begin{table}[htbp]\scriptsize
 \caption{Definitions of distances or divergences from the f-divergence family}
 \begin{tabular}{l|l|l}
  \toprule
   Distance/divergence & Definition  & $f(t)(t=\frac{p_{i}(x)}{p_{j}(x)})$ \\
  \midrule
  Bhattacharyya distance & $-\ln\int\sqrt{p_{i}p_{j}}dx$ & $\sqrt{t}$ \\
  KL-divergence & $\int{p_{i}}\ln\frac{p_{i}}{p_{j}}dx$ & $t\log(t)$ \\
  Symm. KL-divergence & $\int(p_{i}\ln\frac{p_{i}}{p_{j}}+p_{j}\ln\frac{p_{j}}{p_{i}})dx$ & $t\log(t)-\log(t)$ \\
  Hellinger distance & $\frac{1}{2}\int(\sqrt{p_{i}}-\sqrt{p_{j}})^{2}dx$ & $\frac{1}{2}(\sqrt{t}-1)^{2}$\\
  Total variation & $\int|p_{i}-p_{j}|dx$ & $|t-1|$\\
  Pearson divergence & $\int\frac{1}{p_{j}}(p_{i}-p_{j})^{2}dx$  & $(t-1)^{2}$\\
 Alpha divergence & $\frac{1}{\alpha(1-\alpha)}\int(1-\int{p_{i}^{\alpha}p_{j}^{1-\alpha}})dx $ & $ \frac{t}{1-\alpha}+\frac{1}{\alpha}-\frac{t^{\alpha}}{\alpha(1-\alpha)}$\\
  \bottomrule
 \end{tabular}
\end{table}

It has been shown that the \emph{f}-divergence has a number of remarkable properties \cite{41,43}. One of its advantages is that the \emph{f}-divergence between two distributions is invariant to transformations \cite{44}. Consider a feature space $\Re$ and two distributions $p_{i}(x)$ and $p_{j}(x)$ in $\Re$ ($x\in\Re$). Let $g:X \rightarrow Y$ (linear or nonlinear) denotes an invertible transformation function, which maps $x$ into a new feature $y$. By this way, the distributions $p_{i}(x)$ and $p_{j}(x)$ are transformed to $q_{i}(y)$ and $q_{j}(y)$. We aim to seek an invariant measurement $D$ that $D(p_{i},p_{j})=D(q_{i},q_{j})$, based on the following theorem \cite{qiao2008f}.
\newtheorem*{theorem}{Theorem}

\begin{theorem}
 The f-divergence between two distributions is invariant under an invertible transformation $g$ on the space $\Re$.
\end{theorem}
\begin{proof}
 With the invertible transformation $g$, we have $y=g(x)$, so that the distribution $q_{i}(y)$ can be calculated by,
 \begin{small}
\begin{eqnarray}
q_{i}(y)=p_{i}(g^{-1}(y))G(y),
\end{eqnarray}
\end{small}
where $g^{-1}$ denotes the inverse function of $g$, and $G(y)$ is the absolute value of the determinant of Jacobian matrix of the $g^{-1}(y)$. With $dx=G(y)dy$, we have,
 \begin{small}
\begin{eqnarray}
D_{f}(p_{i},p_{j})\nonumber
&=&\int p_{j}(x)f(\frac{p_{i}(x)}{p_{j}(x)})dx \nonumber \\
&=&\int p_{j}(g^{-1}(y))f(\frac{p_{i}(g^{-1}(y))G(y)}{p_{j}(g^{-1}(y))G(y)})G(y)dy \nonumber\\
&=&\int q_{j}(y)f(\frac{q_{i}(y)}{q_{j}(y)})dy \nonumber\\
&=&D_{f}(q_{i},q_{j})
\end{eqnarray}
\end{small}
\end{proof}

\begin{figure*}
\centering
\includegraphics[height=3in,width=5.5in,angle=0]{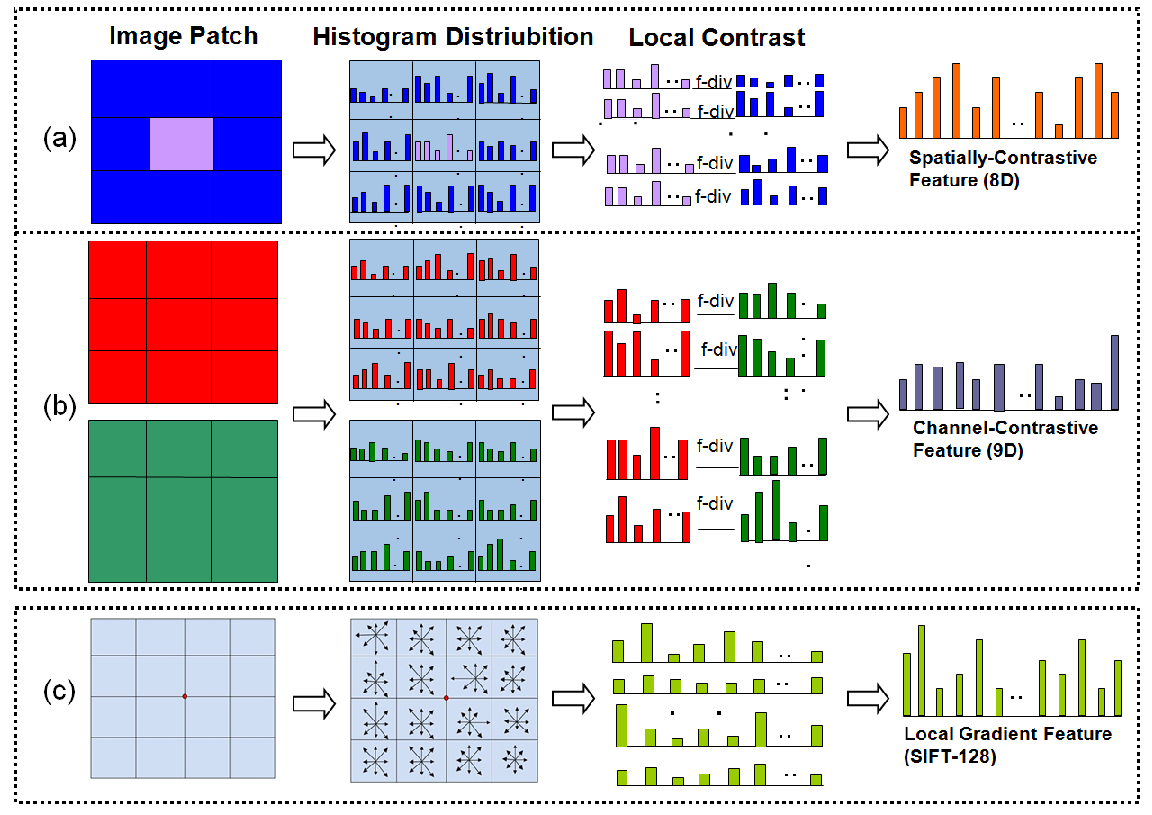}
\caption{Pipeline of the proposed Local Color Contrastive Descriptor (LCCD), including (a) the Spatially-Contrastive Feature (LCCD$_S$), and (b) the Channel-Contrastive Feature (LCCD$_C$). (c) The computation of the SIFT descriptor in a local image path. The main differences between the LCCD and SIFT are demonstrated clearly. The SIFT computes the histograms of gradient information from each divided region independently, while our LCCD captures the spatial correlation feature between neighboring regions (by the spatially-contrastive feature), and encodes color contractive information between different image channels (by the channel-contrastive feature). Thus the LCCD
descriptor provides strong complementary information for the SIFT by leveraging both color and contrastive information of the image.}
\label{fig:12}
\end{figure*}

This is an appealing property for image description, whose goal is to capture meaningful underlying local structure of the image, while being robust against multiple local image distortions. Therefore, we exploit Hellinger distance of the \emph{f}-divergence as the basic function to compute the color contrast between local image regions.

\subsection{Local Color Contrastive Descriptor}
Motivated from the observation in visual neuroscience that contrastive information plays a crucial role on color perception, we aim to enrich the image representation with the contrastive aspect of color and shape information.  To this end, we explore this neural mechanisms of color contrast to design a new and powerful local image descriptor. The proposed Local Color Contrastive Descriptor (LCCD) is computed from an image patch by dividing it into a number of (sub) regions (cells), as shown in the left column of Figure 2. The region-based property of the LCCD increases its robustness against image noise which often affects the performance of the descriptors building on isolated pixel operation, such as the LBP based methods \cite{2,Huang2015,Ojala2002,Ahonen2006}. Besides, it takes into account the spatial layout of the image as well as the statistical properties computed from each region. In order to extract the contrastive features from both spatially neighboring regions and different image channels, the LCCD computes both \emph{spatially-contrastive} and \emph{channelly-contrastive features} based on the \emph{f}-divergence measure.

\subsubsection{Spatially-Contrastive Feature}
  The spatially-contrastive feature computes the relative difference of statistical color features between neighboring regions. The image patch is transformed from the RGB space to the opponent color space  as \cite{35},

   \begin{equation}
\left(
\begin{array}{c}
O_{1}\\
O_{2}\\
O_{3}\\
\end{array}
\right)
=\left(
 \begin{array}{c}
   \frac{R-G}{\sqrt{2}}\\
   \frac{R+G-2B}{\sqrt{6}}\\
   \frac{R+G+B}{\sqrt{3}}\\
  \end{array}
\right)               
\end{equation}
where channel $O_{1}$ and $O_{2}$ include the color information and $O_{3}$ has the intensity feature. We compute a spatially-contrastive feature from each channel. Specifically, as shown in Figure 2(a), an image patch (in one channel) is divided into $3 \times 3=9$ regions. We compute a $d$-bin histogram feature from each region, which can be considered as a discrete probability distribution of the feature from this region.  The contrastive feature is computed by measuring the \emph{f}-divergence between the feature distributions of the central region ($P$) and its 8 neighboring regions ($Q={Q^1,Q^2,...,Q^8}$) (shown in Figure 2(a)),

\begin{small}
\begin{eqnarray}
LCCD_S=[h(P,Q^{1}), h(P,Q^{2}), \ldots, h(P,Q^{8})],
\end{eqnarray}
\end{small}
where $h(P,Q^{i})$ is the Hellinger distance to measure the contrast between two feature histograms from $P$ and $Q^{i}$. It can be computed as,
\begin{small}
\begin{eqnarray}
h(P,Q^{i})=\frac{1}{2}\sum_{k=1}^{d}(\sqrt{p_k}-\sqrt{q^i_k})^{2}, i=1,2,...8,
\end{eqnarray}
\end{small}
where $k$ is the bin index of the $d$-bin histogram, $p_{k}$ and $q_{k}^{i}$ are the values of the $k$-th bins of the feature histograms of the central ($P$) and the $i$-th neighboring regions ($Q^{i}$), respectively.  Therefore,  we compute an 8-dimensional feature vector from an image patch (in a single channel). Each dimension of the vector corresponds to a contrastive value from a neighboring region. Finally, we calculate three such 8-dimensional features from $O_{1}$, $O_{2}$ and $O_{3}$  channels respectively, and concatenate them to construct the final spatially-contrastive feature, with dimensions of 24D. The number of dimensions is significantly lower than 128D used by the SIFT. One may suggest to use a more complex distance function of the \emph{f}-divergence, such as the Alpha divergence. But we found experimentally that other complex distances do not yield a significant better result. The Hellinger distance is simple, but it is effective enough to capture the meaningful local contrastive structure of the images.

\subsubsection{Channel-Contrastive Feature}
 The channel-contrastive feature computes the feature contrast between different channels of a same region. Similarly, the Hellinger distance is used to compute the \emph{f}-divergence between the R, G, and B channels. As shown in Figure 2(b), we first extract a histogram feature from each region in both referred channels. Then we compute the \emph{f}-divergence between two channels as,

\begin{small}
\begin{eqnarray}
h(Q_x^i, Q_y^i)=\frac{1}{2}\sum_{k=1}^{d}(\sqrt{q_{x,k}^i}-\sqrt{q_{y,k}^i})^{2}
\end{eqnarray}
\end{small}
where $Q_x^i$ and $Q_y^i$ are the $i$-th regions of the $x$ and $y$ channels. $k$ is the bin index of the $d$-bin histogram. For  an image patch with $3\times3$ regions, we can get a 9-dimensional channelly-contractive vector as,

\begin{small}
\begin{eqnarray}
LCCD_{C_{xy}}=[h(Q_x^1, Q_y^1),h(Q_x^2, Q_y^2), \ldots, h(Q_x^9, Q_y^9)],
\end{eqnarray}
\end{small}
The final channel-contrastive feature (LCCD$_C$) is constructed by concatenating three channel-contrastive vectors computed between R and G, R and B, G and B channels, respectively. More discussions on the channel-constrictive feature are presented in Section $3.D$. The proposed LCCD
 descriptor is constructed by using both spatially-contrastive and channel-contrastive features.

\subsection{Subspace Extension}

To enhance the discriminative capability of the LCCD, we introduce a subspace based method to compute a more meaningful feature from each image region. The subspace extension allows the LCCD to capture more detailed features from the image, and hence naturally yields extra important information for computing the color contrast, which is the key to discriminativeness.

The subspace feature is computed based on the original histogram vector extracted from each region. Specially, assume that we have a $d$-bin histogram, the subspace vector is generated by moving a (1D) sub-window of size (length) $d_1$ densely through the original $d$-bin histogram. By this way, the original histogram is now decomposed into multiple subspaces or sub histograms, each of which includes $d_1$ bins. The number of the newly generated subspaces is $d-d_1+1$.

We compute each subspace contrast between two generated sub histograms independently. The Eq. (4) is extended as,
\begin{small}
\begin{eqnarray}
h_{sub}^j(P,Q^{i})=\frac{1}{2}\sum_{k=j}^{j+d_1-1}(\sqrt{p_k}-\sqrt{q^i_k})^{2},
\end{eqnarray}
\end{small}
where $ j=1,2,...,(d-d_1+1).$
Then the final contrastive feature ($h_{sub}(P,Q^{i})$)  is constructed by concatenating all the subspace contrasts ($h_{sub}^j$) computed between two considered regions (e.g. region $P$ and $Q^i$),
\begin{footnotesize}
\begin{eqnarray}
h_{sub}(P,Q^{i})=[h_{sub}^1(P,Q^{i}),h_{sub}^2(P,Q^{i}),\ldots,h_{sub}^{d-d_1+1}(P,Q^{i})],
\end{eqnarray}
\end{footnotesize}
Therefore, the subspace-extended contrastive feature between two regions is a $(d-d_1+1)$-dimensional vector, while the original non-subspace one is a single value computed by Eq. (4). The subspace extension is involved in region level, hence it can be directly adopted to compute both spatially- and channel- contrastive features.  In all our experiments, we empirically set the size of subspace window to $d_1=3$.

\subsection{Analysis and Discussions}
\begin{figure}
\centering
\includegraphics[height=1.2in,width=3.2in,angle=0]{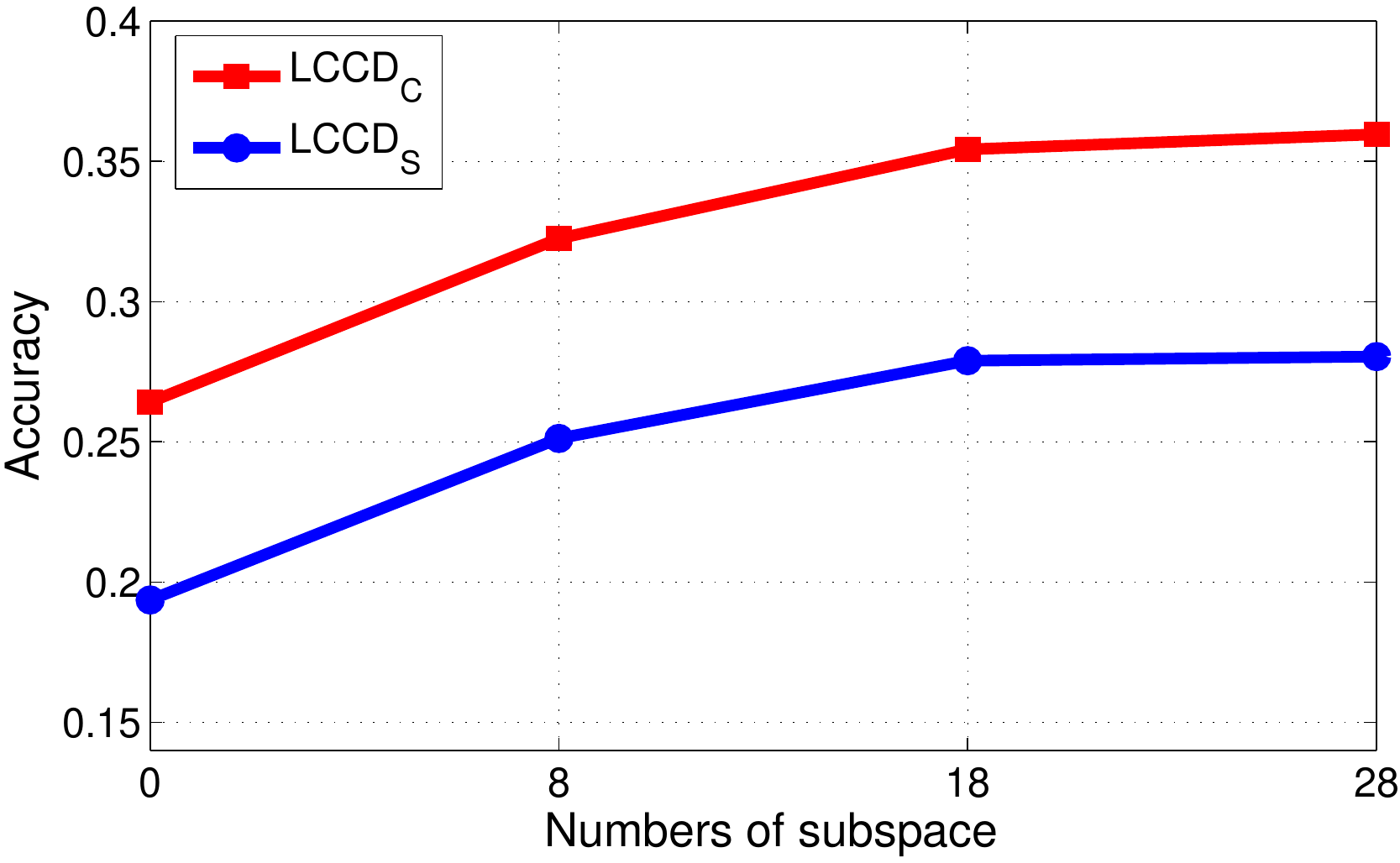}
\caption{Performance of the spatially-contrastive feature (LCCD$_S$) and channel-contrastive feature (LCCD$_C$) with non-subspace and various dimensions of the subspace extension (on the MIT Indoor-67 database).}
\label{fig:12}
\end{figure}

To verify efficiency of both spatially-contrastive and channel-contrastive features, and the subspace extension, we utilize the MIT Indoor-67 database \cite{16} (the details of the database are described in Section 4) to compare the performance of the LCCD$_S$ and LCCD$_C$ with non-subspace and subspace extensions under various dimensions.  Notice that the dimension of the LCCD  with the subspace extension (for an image patch) is determined by the number of the original histogram bins ($d$), and the size (length) of its subspace window ($d_1$): $d_{sub}=d-d_1+1$. In our experiments, we set $d=10, 20, 30$ with a fixed size of the subspace window, $d_1=3$. Then we get $d_{sub}=8, 18, 28$.  The results of both LCCD$_S$ and LCCD$_C$ are presented in Figure 3. It can be found that, the LCCD with subspace extensions consistently outperforms that of the non-subspace ones considerably in both spatially- and channel-constrictive features. The LCCD$_C$ performs slightly better than the LCCD$_S$.  The accuracies increase with increasing numbers of the subspace dimensions, and reach their stabilities at 18-dimensional subspaces in both cases. By trading off the performance and computational cost, we use the 18-dimensional subspace in all our following experiments.

We further investigate the performance of the LCCD$_S$ and  LCCD$_C$ separately in Figure 4(a), and their combinations with SIFT in Figure 4(b). In the Figure 4(a), the LCCD$_C$ computed from the RG and RB channels receive slightly higher performance than that of the GB channels, and the combination of three channel contrasts receives a significant improvement. The LCCD$_S$ gets a slightly better accuracy than the LCCD$_C$ on each single contrast, but its performance is lower than that of the LCCD$_C$ with the combination of three contrasts. As expected, the combination of both the LCCD$_S$ and LCCD$_C$ achieves a further improvement over each single performance.

In Figure 4(b), we can find that either single LCCD$_S$ or LCCD$_C$ can improve the performance of the original SIFT largely by combination, and the highest accuracy is obtained  by combining the SIFT with both of them. As can be found, for the combination with SIFT, the LCCD$_C$ by using only RG and RB channels achieves a slightly higher accuracy than the LCCD$_C$ using three channel contrasts. We obtain similar results when we conducted more experiments on other databases in the Section 4. We guess that the contrastive feature included between the GB channels may be relatively weak or highly redundant. Therefore, our final LCCD descriptor applied in all our following experiments only contains the LCCD$_S$, and the LCCD$_C$ computed from the RG and RB channels.


\begin{figure}
\centering
\includegraphics [width = 4.3 cm,height=3.5cm]{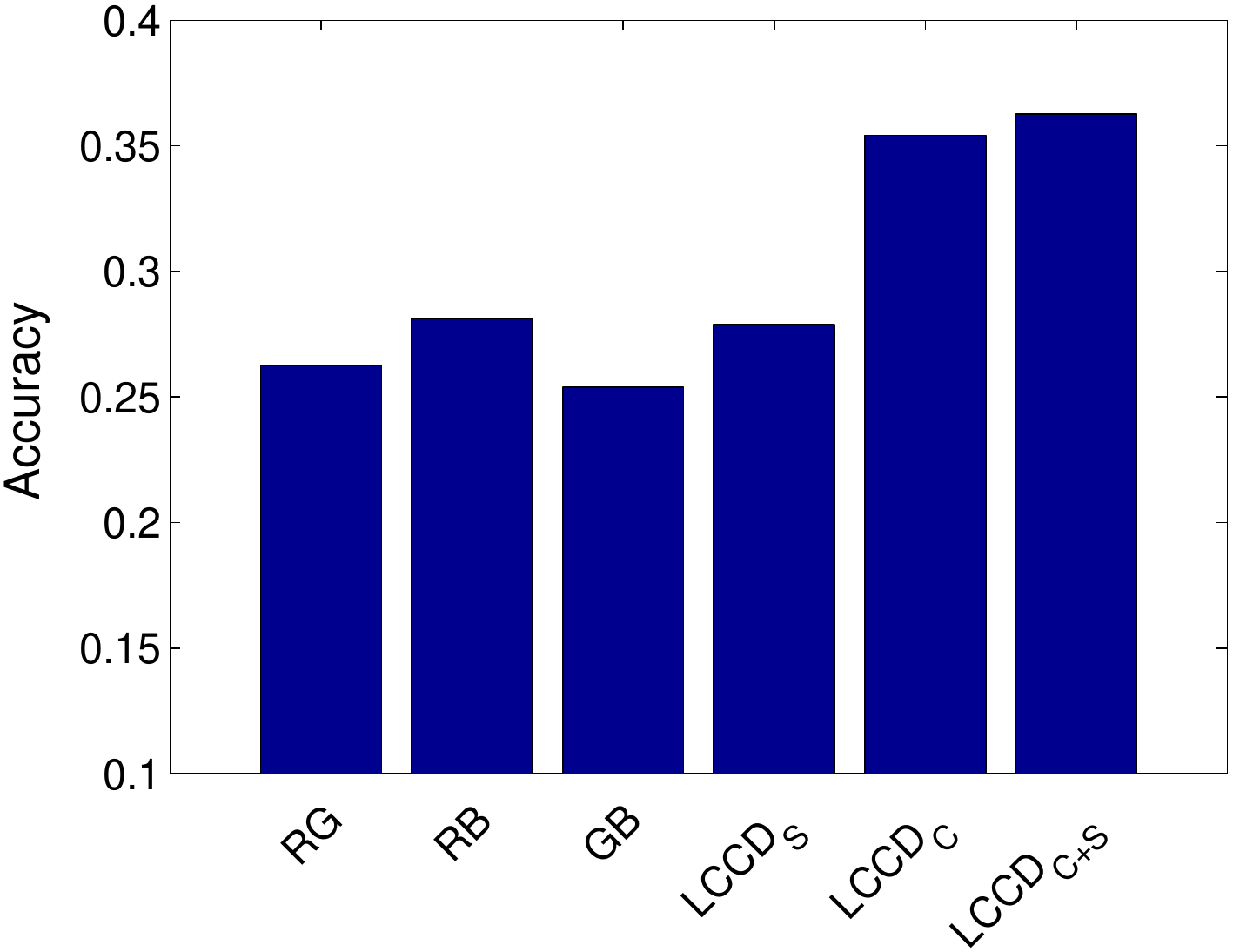}
\includegraphics [width = 4.3 cm,height=3.5cm]{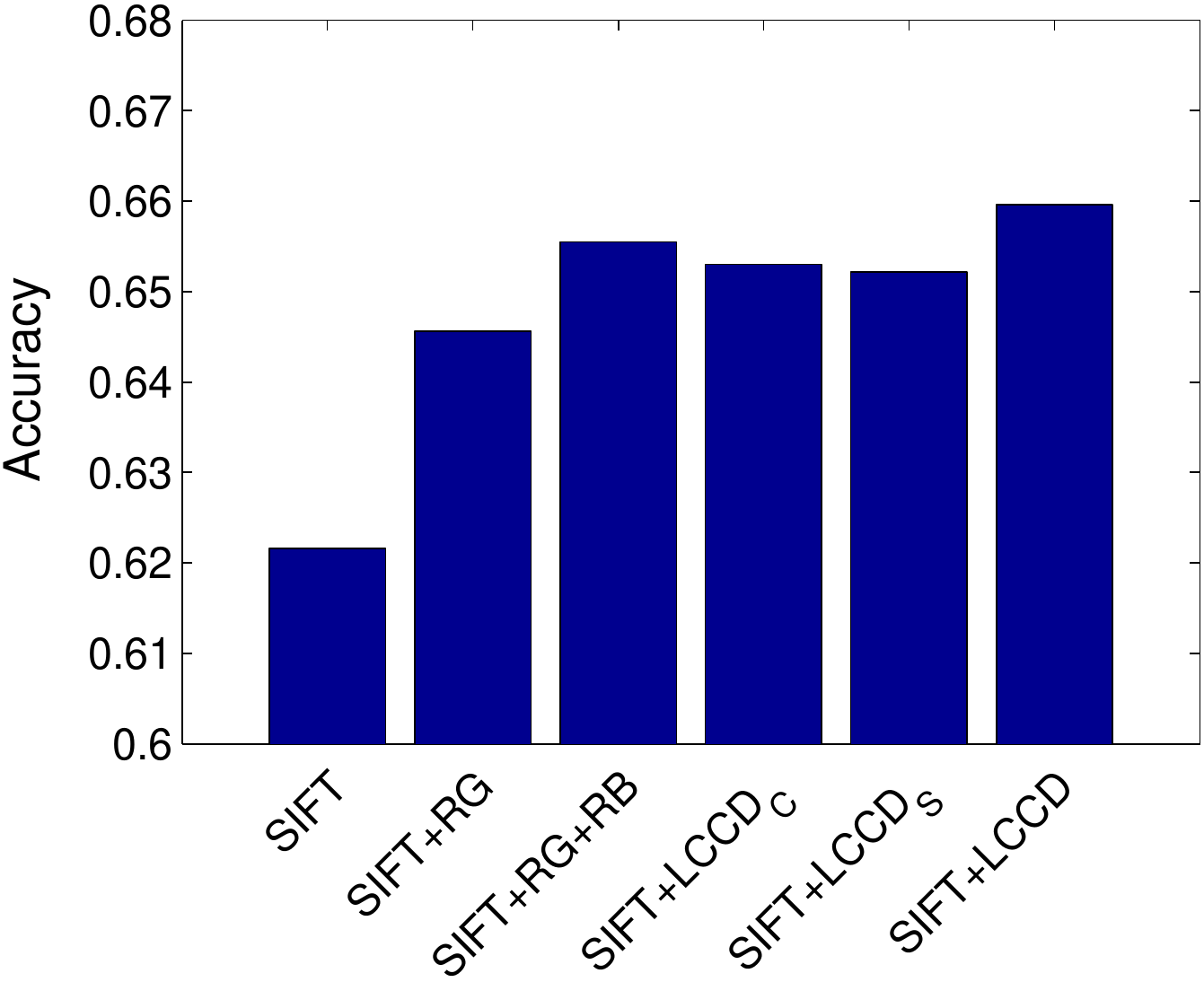}
\caption{(a) Performance of the LCCD$_S$ and LCCD$_C$ in single RG, RB and GB channels, and their combinations, all with subspace extension. (b) Combinations of the LCCD$_S$ and/or LCCD$_C$ with the SIFT.}
\label{fig:12}
\end{figure}

We discuss the fundamental difference between the proposed LCCD and SIFT \cite{1}, which has received great success for image description in last decade. The SIFT is extremely powerful for detecting robust shape information of the image by computing local gradient orientations. For comparing the LCCD with SIFT, we present the basic pipeline of the SIFT in Figure 2(c). The SIFT descriptor (for an image patch) is generated by concatenating 16 histograms, each of which is computed on gradient orientations from a divided region independently. It can be found clearly that the LCCD is different from the SIFT at two main aspects. Firstly, the SIFT is computed in the gradient space by applying the histogram of gradient orientation for feature representation which only includes main shape information of the image. The LCCD is able to explore meaningful color information as an important complementary feature that enriches the representation. Secondly, the LCCD computes multiple contrastive features both spatially and between multiple color channels, making it capable of encoding more meaningful local contextual and underlying structural information than the SIFT, which does not consider local spatial relationships (e.g. local contrast) between neighboring regions at all. This may lead to a significant information loss of the SIFT. We will show experimentally that both advanced properties of the LCCD provide strong  complementary to SIFT descriptor for image classification.

We further show the insights of the proposed contrastive mechanism by connecting it with recent Deep Convolutional Neural Networks (DCNN) \cite{decaf2013}. Our observation can be verified by recent success of the DCNN, which shows that both color contrast and edge information are the main low-level features for image description, as indicated in Figure 1. Intuitively, the design of our color contrastive mechanism is closely related to the structures of  low-level filters (from the first convolutional layer) learned by the DCNN. As shown in Figure 1, there are mainly two types of the low-level filters, and one of them intuitively corresponds to our color contrast mechanism (mainly on the bottom part) as follow. First, some filters are mainly displayed in a single color. It means that their weights are varied largely between image channels, but are changed slightly within each independent channel. Hence, they are able to capture the contrastive characteristics between different color channels. This mechanism is relatively closed to the pipeline of our channel-contractive feature. Second, some other filters are displayed in two or multiple colors, indicating that their weights are changed significantly both spatially and between channels. Therefore, they are able to detect the  contrastive information from both aspects, which are similar to both of our contrastive descriptors. These intuitive low-level connections between the LCCD and DCNN provides a strong theoretical support to the proposed color constrictive mechanism.

\section{Experimental Results and Discussions}
The performance of the LCCD was evaluated on three challenging benchmark databases for image classification and scene categorization: the MIT Indoor-67 database \cite{16}, SUN397 \cite{xiao2010sun} and PASCAL VOC 2007 standards\cite{29}. We compare the performance of the LCCD and its combination with SFIT against recent results on three databases.

In all our experiments, we resize the input image into $470 \times 380$. Each image is divided into $50 \times 50$ regions (cells). A LCCD feature vector is extracted from an image patch with the size of $3 \times 3$ regions.  The LCCD features are computed densely by moving a patch window with the size of $3 \times 3$ regions through all the $50 \times 50$ divided regions. Finally, we get $48 \times 48$ LCCD feature vectors from an image. Each LCCD feature vector is computed as follow. First, a 20-bin histogram feature is extracted from each region for computing the color contrast.  Second, the subspace scheme is adopted, and the 20-bin histogram is decomposed into 18 subspaces or sub histograms by using the subspace window with size (length) of $3$. Third, we compute a contrastive value from each pair of subspaces by using Eq. (6) or Eq. (7), and then get an 18D  contrastive vector from each pair of considered regions. Fourth, for a defined image patch, we compute the contrastive vectors from all possible region pairs, and generate the final spatially- and channel-contrastive features with dimensions of $18 \times 8=144$, and $18 \times 9=162$, respectively. Then we further reduce the dimensions of both contrastive features to 80 by using PCA \cite{14}, and finally generate  LCCD$_S$, LCCD$_C\in\mathbb{R}^{80\times 48\times 48}$ for an image.

We apply Fisher Vector (FV) encoding for both LCCD$_S$ and LCCD$_C$ separately. We train a codebook with 256 centers using the Gaussian Mixture Model (GMM), and encode the generated LCCD$_S$ or LCCD$_C$ vectors with the BoW model \cite{wang2013comparative,yang2010bag,28}. The final LCCD descriptor combines both LCCD$_S$ and LCCD$_C$. For SIFT, we used the VLFeat \cite{12} library to extract SIFT descriptor \cite{1} with 128 dimensions for each patch. Similarity, they are reduced into 80D by using PCA \cite{14}, and then are also encoded with the BoW model with 256 centers.

\subsection{On the MIT Indoor-67 Database}

We evaluate the performance of the LCCD on the task of indoor scene recognition. The experiments were conducted on the large-scale MIT Indoor-67 database \cite{16}, which contains 67 classes and total 15,620 images. The number of images varies across categories, but at least 100 images are included in each category. The numbers of training and  testing images are 80 and 20 per category, respectively.

\begin{table}[!h]
\caption{Comparisons of the LCCD with the-state-of-the-art descriptors WITHOUT Fisher Vector Encoding on the MIT Indoor-67.}
\centering  
\renewcommand{\arraystretch}{1.2}
\begin{tabular}{p{3.5cm}||p{2cm}<{\centering}||p{2cm}<{\centering}}
\hline
\textbf{Method} & \textbf{Publication} &\textbf{Accuracy}($\%$)\\
\hline
Quattoni \textit{et.al.}\cite{16} &CVPR2009 & 26.00\\
\hline
Li \textit{et.al.} \cite{7}& NIPS2010 &37.60\\
\hline
Wang \textit{et.al.} \cite{17}& CVPR2010  & 54.62\\\hline

\hline
SIFT \cite{1}& IJCV2004 &51.85\\
Color SIFT \cite{35}& TPAMI2010 &56.10\\
BoP+SIFT(Juneja \textit{et.al.}\cite{8})& CVPR2013 & 56.66\\
\hline
\hline
LCCD & -- & 20.36\\
LCCD+SIFT & -- & \textbf{57.42}\\
\hline
\end{tabular}
\end{table}
\begin{table}[!h]
\caption{Comparisons of the LCCD with the-state-of-the-art descriptors WITH Fisher Vector Encoding on the MIT Indoor-67.}
\centering  
\renewcommand{\arraystretch}{1.2}
\begin{tabular}{p{3.5cm}||p{2cm}<{\centering}||p{2cm}<{\centering}}
\hline
\textbf{Method} & \textbf{Publication} & \textbf{Accuracy}($\%$)\\
\hline

Kobayashi \emph{et.al.} \cite{19} & CVPR2013&58.91\\
\hline
Doersch \textit{et.al.} \cite{20} &NIPS2013 &64.03\\

\hline
\hline
SIFT (Jorge \textit{et.al.}) \cite{28} &  ECCV2010 &62.16\\
Color SIFT \cite{35}& TPAMI2010 &64.22\\
BoP+SIFT(Juneja \textit{et.al.} \cite{8}) &CVPR2013  & 63.10\\

OPM+SIFT(Xie \textit{et.al.} \cite{10}) &CVPR2014 & 63.48\\
\hline
\hline
LCCD & -- & 36.43\\
LCCD+SIFT & -- & \textbf{65.96}\\
\hline
\end{tabular}
\end{table}
It has been verified that the FV encoding can improve the performance considerably on this database. For a fair comparison, we conducted two groups of experiments separately by evaluating the methods with and without the FV encoding. The results of them are presented in Table 2 and 3, respectively.

\begin{figure}
\centering
\includegraphics[scale=0.27]{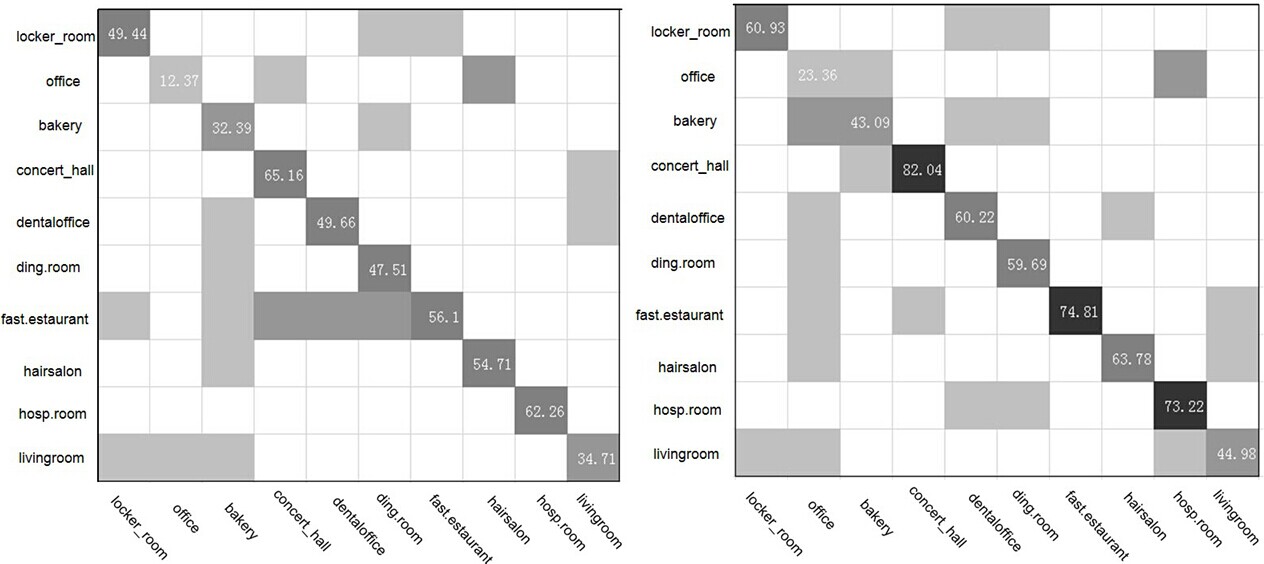}
\caption{Confusion matrices for the SIFT (left), and LCCD+SIFT (right) on ten categories.}
\label{fig:12}
\end{figure}

As can be found, our LCCD descriptor combined with SIFT achieves the best performance in both cases. In the case of with FV encoding, the LCCD+SIFT achieves classification accuracy at $65.96\%$, which surpasses the closest result achieved by the OPM+SIFT \cite{10}) by a large margin of about $2.5\%$. Obviously, it improves the performance of only SIFT substantially in both cases, with improvements at about $6\%$ and $4\%$ for without and  with VF. These improvements are considerably larger than those done by recent proposed  Bag-of-Parts (BoP) \cite{8} and  Orientational Pyramid Matching (OPM) \cite{10} descriptors, which improve SIFT with $0.94\%$ and $1.32\%$ respectively in with FV case. It clearly indicates that our color contrastive feature provides stronger complementary information for SIFT than the BoP and OPM methods.
Beside, we also compared our descriptor against the color SIFT and obtained about $1.5\%$ improvements in both cases, demonstrating that our measure of color in contrast mechanism is more efficient than the color feature applied in the color SIFT descriptor.

To find more detailed linkage between the LCCD and SIFT, we construct two confusion matrices for SIFT and SIFT+LCCD by using ten categories: $bakery$, $concert hall$, $dental office$, $dinging room$, $hairsalon$, $hospital room$, $fast$ $food restaurant$, $office$, $livingroom$ and $locker room$. The two matrices are shown in Figure 5.  Obviously, the values of  diagonal elements in the SIFT+LCCD  matrix are significantly larger than those in the single SIFT matrix. For example, the accuracies of the $concert hall$ and $fastfood restaurant$ increase substantially:  $65.16\%\rightarrow82.04\%$ and  $56.10\%\rightarrow74.81\%$ respectively, indicating that the LCCD descriptor is greatly complementary to  SIFT for image description.

\begin{table}[!h]
\caption{Classification errors within paired categories by single LCCD or SIFT or combination of them.}
\centering  
\renewcommand{\arraystretch}{1}
\begin{tabular}{p{1.8cm}|p{1.8cm}||p{1.5cm}<{\centering}|p{1.5cm}<{\centering}}
\hline
category A & category B & SIFT & LCCD+SIFT \\\hline
$studionmusic$ & $tvstudio$ & 10.53\% & 5.27\% \\
$restaurant$ & $bar$ & 5\% & 0 \\
$poolinside$ & $airportinside$ & $10\%$ & $5\%$ \\
$clothingstore$ & $bedroom$ & $5.56\%$ & 0 \\
$corridor$ & $stairscase$ & $9.52\%$ & $4.76\%$\\
\hline
\hline
category A & category B & LCCD & LCCD+SIFT \\\hline
$gym$ & $closet$ & 5.56\% & 0\\
$hairsalon$ & $greenhouse$ & 4.76\% & 0\\
$jewellery$ & $mail$ & 9.09\% & 4.55\% \\
$meetingroom$ & $classroom$ & 13.64\% & 4.55\% \\
$restaurant$ & $buffet$ & 5\% & 0 \\ \hline
\end{tabular}
\end{table}

\begin{figure*}
\centering
\includegraphics[height=4in,width=6in,angle=0]{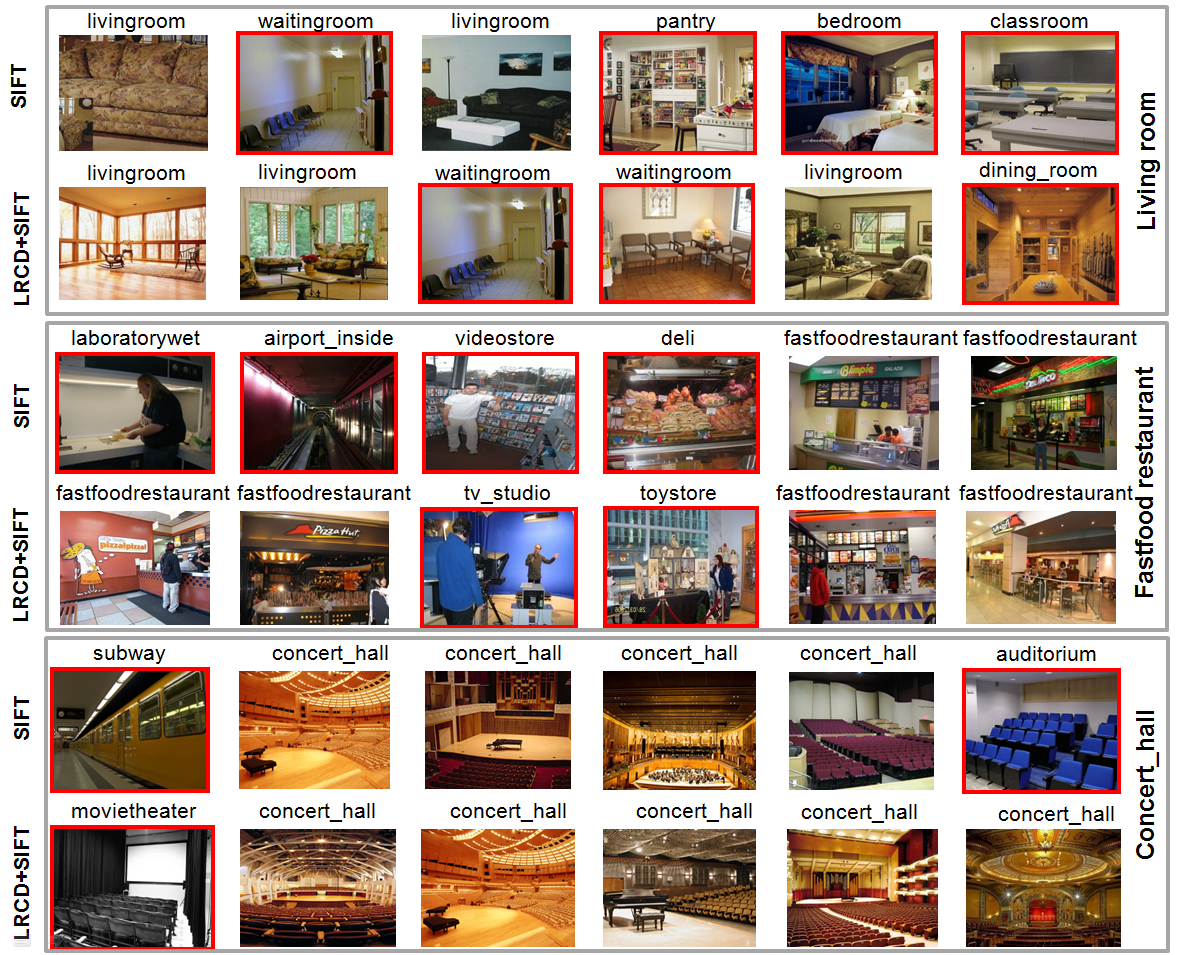}
\caption{ \footnotesize Image samples from categorization results by the SIFT and LCCD+SIFT. The name on top of each image denotes the ground truth category. Images from left to right are sorted by their precision scores in decreasing order from 10th to 15th. Images with top 9 precision scores are almost correctly classified.  Images with incorrect classification are labeled by red boundary boxes.}
\label{fig:12}
\end{figure*}

To further evaluate efficiency of the LCCD, we select several pairs of categories which are difficult to be classified correctly by either the single SIFT or LCCD. The error rates by each of them, and their combination are listed in Table 4. It can be found clearly that the error rates are reduced largely (about $5\%$)  by the combination of them, some of which achieve perfect performance with zero errors, further indicating that the LCCD and SIFT compensate well for each other. As a better demonstration, we also present  a number of example images categorized by SIFT and LCCD+SIFT in Figure 6. The improvements by our descriptor are obvious again. Most incorrect categorizations by our descriptor are acceptable, since most of these cases are even hard to be separated correctly by our human, such as $living room$ and $waiting room$,  $movie theater$ and $concert hall$.


\subsection{On the SUN397}
The performance of the proposed LCCD descriptor was evaluated on the SUN397 \cite{xiao2010sun}. The database has 397 different scene classes, which is probably the largest database for scene classification until now. It includes 108,754 images in total. The number of images varies across classes, and at least 100 images are included in each class. Our experiments follow previous work \cite{28,10} by using a subset of the dataset, which has 50 training and 50 testing images per class, averaging over 10 partitions.

\begin{table}[!h]
\caption{Comparisons of the LCCD with the-state-of-the-art descriptors WITH Fisher Vector Encoding on the SUN397.}
\centering  
\renewcommand{\arraystretch}{1.3}
\begin{tabular}{p{3.8cm}||p{1.7cm}<{\centering}||p{2cm}<{\centering}}
\hline
\textbf{Method} & \textbf{Publication} & \textbf{Accuracy}($\%$)\\
\hline

Xiao \emph{et.al.}\cite{xiao2010sun} & CVPR2010&38.00\\
DeCAF \cite{decaf2013} & ICML2014 & 40.94\\
\hline
SIFT(Jorge \textit{et.al.}) \cite{28} &  IJCV2013 &43.02\\
LCS+SIFT(Jorge \textit{et.al.}) \cite{28}& IJCV2013 &47.20\\
OPM+SIFT (Xie \textit{et.al.} \cite{10}) &CVPR2014 & 45.91\\
\hline
\hline
LCCD & -- & 20.29\\
LCCD+SIFT & -- & \textbf{49.68}\\
\hline

%
\end{tabular}
\end{table}

The performance of the LCCD descriptor was  evaluated with the FV encoding by comparing it with recent results in the SUN397 database. As shown in Table 5, LCCD+SIFT descriptor achieves the highest mean accuracy at $49.68\%$, which largely improves the performance of single SIFT with more than $6.5\%$. The improvement is more significant than the most recent combination methods by the OPM+SIFT (at $45.91\%$) \cite{10} and LCS+SIFT (at $47.20\%$) \cite{28}. Several image categories with top improvements by our combined descriptor, compared to the single SIFT, are presented in Figure 7. It can be found that our descriptor boosts the performance of the SIFT substantially, with improvements of $30\%$ in $swimming$ $pool$  $outdoor$ and $20\%$ in $thrifshop$ categories.  In the right of the Figure 7, we list a number of categories which have very similar global structures to the left ones, making them difficult to be discriminated correctly. Such confused categories commonly exist in scene recognition, some of them in the MIT Indoor-67 are also shown in Figure 6. Our improvements on these categories demonstrate that our color descriptor is able to capture more local detailed features which are crucial to identify these ambiguous categories.

Furthermore, we notice that the proposed LCCD+SIFT descriptor also obtains large improvement (about $9\%$) over recent result of the DeCAF \cite{decaf2013}, which is one of the most advanced deep learning models. These results convincingly verify the effectiveness of the proposed LCCD. Deep learning models have shown strong capability for image representation. However, the high-level deep features computed via multi-layer feature extraction are highly abstracted. They may lose important local detailed information in fully-connected layers, leading to the lower discrimination of the features.

\begin{figure}
\centering
\includegraphics[scale=0.45]{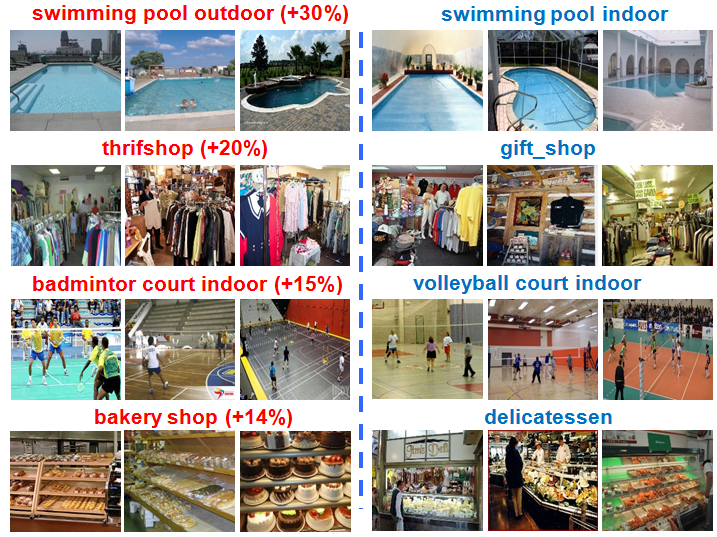}
\caption{Left: Image classes from the SUN397 where the LCCD+SIFT achieves large improvements over SIFT; Right: Image classes which are easily confused with the left ones.}
\label{fig:12}
\end{figure}

\subsection{On the PASCAL VOC 2007 Standards}
\par
We further evaluate the performance of LCCD descriptor on the PASCAL VOC 2007 standards \cite{29} for visual object categorization. The PASCAL VOC 2007 standards \cite{29} is known as one of the most difficult image classification tasks due to large-scale variations in appearance, posture, and even with occlusions, which are often caused by real-world complicated affects. In the table 6, we compare the classification accuracies of related descriptors. All results are achieved with $FV$ encoding, except for color SIFT (cited from \cite{35}). Again, the LCCD+SIFT achieves the highest accuracy at $65.80\%$, improving the performance of individual SIFT (with the FV) with $4\%$. The improvement is more significant than that of the LCS+SIFT descriptor \cite{28}. The advantage of the color contrastive information is obvious again.

\begin{table}[!h]
\caption{Comparison of our results to the state-of-the-art mean accuracies on the  PASCAL VOC 2007 dataset}
\centering  
\renewcommand{\arraystretch}{1.3}
\begin{tabular}{p{3.8cm}|p{1.7cm}<{\centering}|p{2cm}<{\centering}}
\hline
\textbf{Method} & \textbf{Publication} & \textbf{ mAP}($\%$)\\
\hline
Chatfield\textit{et.al.} \cite{25} & BMVC2011  &61.70\\
\hline
Rassakovsky \cite{26} &ECCV2012 &57.20\\
\hline
Kobayashi \textit{et.al.} \cite{19} & CVPR2013 &62.20\\
\hline
Wu \textit{et.al.} \cite{wu2013scale} &CVPR2013  & 64.10 \\
\hline
\hline
SIFT(Jorge \textit{et.al.}) \cite{28} &  IJCV2013 &61.80\\
Color SIFT \cite{35}& TPAMI2010 &42.00\\
LCS+SIFT (Jorge \textit{et.al.}) \cite{28}& IJCV2013 &63.90\\
\hline
\hline
LCCD & -- & 42.45\\
LCCD+SIFT & -- & \textbf{65.80}\\
\hline
\end{tabular}
\end{table}

\subsection{Summary}
We conducted extensive experiments on three benchmarks to evaluate the performance of the LCCD descriptor. The results are  summarized as follows. First, although the performance of single LCCD descriptor is generally not comparable to that of SIFT,  the combination of them consistently improves the performance of individual SIFT substantially in all our experiments. Second, the LCCD+SIFT achieves the-state-of-the-art performance in all three databases. It outperforms recent combination descriptors (e.g. the LCS+SIFT \cite{28} and OPM+SIFT \cite{10}) considerably, indicating that the proposed LCCD provides stronger complementary properties to the SIFT than the others. The LCCD is capable of capturing meaningful local detailed features, which are often discarded by most gradient based descriptors and DCNN models. Third, our computation of color contrast in both spatial locations and multiple channels achieves better performance than current color SIFT and LCS descriptors, demonstrating that our contrastive mechanisms provides a more principled approach for measuring local color information.

\section{Conclusions}
\par
We have presented a simple yet powerful local descriptor, local color contrastive descriptor (LCCD), for image classification. Beyond traditional shape based descriptor, the neural mechanisms of color contrast was introduced to enrich the image representation with color information in multimedia and computer vision communities.  We developed a novel contrastive mechanism to compute the color contrast in both spatial locations and multiple color channels, and successfully applied it for detecting meaningful local structures of the images. We verified its efficiency both theocratically and experimentally, and demonstrated its strong ability to compensate for the SIFT feature for image description. Extensive experimental results show that the proposed LCCD descriptor with SIFT substantially improves the performance of individual SIFT, and achieves the-state-of-the-art performance in three benchmarks, verifying its efficiency convincingly and confidently.

%
%
%
%
%
%
%
%

\bibliographystyle{IEEEtran}
\bibliography{mybibfile}

\end{document}